%% file: Main.tex
\documentclass[letterpaper, 10 pt, conference]{ieeeconf}  % Comment this line out if you need a4paper

\IEEEoverridecommandlockouts                              % This command is only needed if 
\usepackage{graphicx} % for pdf, bitmapped graphics files
\usepackage{epsfig} % for postscript graphics files
\usepackage{amsmath} % assumes amsmath package installed
\usepackage{amssymb}  % assumes amsmath package installed
\usepackage{subfig}
\usepackage[ruled,longend]{algorithm2e}
\usepackage{multicol}
\usepackage{color}

\newtheorem{proposition}{Proposition}
\newtheorem{lemma}{Lemma}
\newtheorem{remark}{Remark} 
\title{\LARGE \bf
 Decoupled Data Based Approach for Learning to Control Nonlinear Dynamical Systems
}

\author{Ran Wang$^1$,  Karthikeya Parunandi$^1$, Dan Yu$^2$, Dileep Kalathil$^3$, Suman Chakravorty$^1$
\thanks{$^{1}$R. Wang, K. Parunandi and S.  Chakravorty are with the Department of Aerospace Engineering,  Texas A\&M University, Texas, USA. 
        {\tt\small \{rwang0417@tamu.edu, s.parunandi, schakrav\}@tamu.edu}}%
\thanks{$^{2}$D.Yu is with  the College of Astronautics,
        Nanjing University of Aeronautics and Astronautics, Nanjing, 210016, China.
        {\tt\small yudan@nuaa.edu.cn}}%
\thanks{$^{3}$D. Kalathil is with the Department of Electrical and Computer Engineering,  Texas A\&M University, Texas, USA. 
        {\tt\small  dileep.kalathil@tamu.edu}}%
}

\begin{document}

\maketitle
\thispagestyle{empty}
\pagestyle{empty}

\begin{abstract}
This paper addresses the problem of learning the optimal control policy for a nonlinear stochastic dynamical system  with continuous state space, continuous action space and unknown dynamics.  This class of problems are typically addressed in stochastic adaptive control and reinforcement learning literature using  model-based and model-free approaches respectively. Both methods rely on solving a dynamic programming problem, either directly or indirectly, for finding the optimal closed loop control policy. The inherent `curse of dimensionality' associated with  dynamic programming method makes these approaches also computationally difficult.    

This paper proposes a novel  decoupled data-based control (D2C) algorithm that addresses this problem using  a decoupled,  `open loop - closed loop', approach.  First, an open-loop deterministic trajectory optimization problem is solved using a black-box simulation model of the dynamical system. Then, a closed loop control is developed around this open loop trajectory by linearization of the dynamics about this nominal trajectory. By virtue of linearization, a linear quadratic regulator based algorithm can be used for this closed loop  control. We show that the performance of D2C algorithm is approximately optimal. Moreover, simulation performance suggests significant reduction in training time compared to other state of the art algorithms.   
\end{abstract}

\input{Introduction.tex} 
\input{Problem_Formulation.tex}

\input{A_Near_Optimal_Decoupling_Principle.tex}
\input{Algorithm.tex}

\input{Simulations.tex}
\input{Conclusions.tex}

\input{Bibliography.tex}
\input{Appendix.tex}

\end{document}

%% file: Introduction.tex
\section{Introduction}\label{sec1}
Controlling an unknown dynamical system adaptively has a rich history in control literature \cite{kumar2015stochastic} \cite{ioannou2012robust}. These classical literature provides rigorous  analysis about the asymptotic performance and  stability of the closed loop system.  Classical adaptive control literature mainly focuses on non-stochastic systems \cite{aastrom2013adaptive} \cite{sastry2011adaptive}. Stochastic adaptive control literature mostly addresses  tractable models like linear quadratic regulator (LQR)   where Riccati equation based  closed form expressions are available for the optimal control law. Optimal control of an unknown nonlinear dynamical system with continuous state space and continuous action  space is a significantly more challenging problem. Even with a known model, computing an optimal control law requires solving a dynamic programming problem. The `curse of dimensionality' associated with dynamic programming  makes solving such problems computationally intractable, except under special structural assumptions on the underlying system. \emph{Learning to control} problems where the model of the system is unknown also suffer from this computational complexity issues, in addition to the usual identifiability problems in adaptive control.       

Last few years have seen significant progresses in deep neural netwoks based reinforcement learning  approaches for controlling unknown dynamical systems, with applications in many areas like playing games \cite{silver2016mastering}, locomotion \cite{lillicrap2015continuous} and robotic hand manipulation \cite{levine2016end}. A number of new algorithms that show promising performance are  proposed \cite{acktr} \cite{trpo} \cite{ppo} and various improvements and innovations have been continuously developed. However, despite excellent performance on a number of tasks, reinforcement learning (RL) is still considered very data intensive. The training time for such algorithms are typically really large. Moreover,  high variance and reproducibility issues on the performance are also reported \cite{henderson2018deep}. While there have been some attempts to improve the sample efficiency \cite{gu2016q}, a systematic approach is still lacking.

In this work, we propose a novel  decoupled data based control (D2C) algorithm for learning to control an unknown nonlinear dynamical system. Our approach introduces a rigorous decoupling of the open loop (planning) problem from the closed loop (feedback control) problem. This decoupling allows us to come up with a highly sample efficient approach to solve the problem in a completely data based fashion. Our approach proceeds in two steps: (i) first, we optimize the nominal open loop trajectory of the system using a blackbox simulation model, (ii) then we identify the linear system governing perturbations from the nominal trajectory using random input-output perturbation data, and design an LQR controller for this linearized system.  We show that the performance of D2C algorithm is approximately optimal,   in the sense that the decoupled design is near optimal to second order in a suitably defined noise parameter. Moreover, simulation performance suggests significant reduction in training time compared to other state of the art algorithms.

{\bf Related works:}  The solution approaches  to the problem of controlling an unknown dynamical systems can be divided into two broad classes, model-based methods and model-free methods.  

 In the model-based methods, many techniques \cite{dp_num} rely on a discretization of the underlying state and action space, and hence, run into the curse of dimensionality, the fact that the computational complexity grows exponentially with the dimension of the state space of the problem.  The most computationally efficient among these techniques are trajectory-based methods such as  differential dynamic programming (DDP) \cite{ddp} \cite{sddp} which linearizes the dynamics and the cost-to-go function around a given nominal trajectory, and designs a local feedback controller using DP. The iterative linear quadratic Gaussian  (ILQG) \cite{ilqg1}  \cite{ilqg2}, which is closely related to DDP,  considers the first order expansion of the dynamics (in DDP, a second order expansion is considered), and designs the feedback controller using Riccati-like equations, and is shown to be computationally more efficient. In both approaches, the control policy is executed to compute a new nominal trajectory, and the procedure is repeated until convergence. 
 
 %Alternatively, a Trajectory-optimized LQG (T-LQG) approach \cite{tlqg,Separation} was recently proposed which shows that under a first order approximation of the dynamics, and cost-to-go function,  a near optimal solution can be found by first solving a deterministic trajectory optimization problem, followed by a linear time-varying closed-loop controller design problem.  This separated approach can also be extended to the model free case and is the subject of the current paper:  we use a gradient descent algorithm, and a linear time varying (LTV) system identification technique, in conjunction with a black box simulation model of the process, to accomplish the ``separated" design.

Model-free methods, more popularly known as  approximate dynamic programming   \cite{powell2007approximate} \cite{bertsekas2012dynamic} or  reinforcement learning (RL) methods \cite{sutton2018reinforcement}, seek to improve the control policy  by repeated interactions with the environment, while observing the system's responses. The repeated interactions, or learning trials, allow these algorithms to compute the solution of the dynamic programming problem (optimal value/Q-value function or optimal policy) without explicitly constructing the model of the unknown dynamical system.  Standard RL algorithms are broadly  divided into value-based methods, like Q-learning, and policy-based methods, like policy gradient algorithms. Recently, function approximation using deep neural networks has significantly improved the performance of reinforcement learning algorithm, leading to a growing class of literature on `deep reinforcement learning'. Despite the success, the amount of samples and training time required still seem prohibitive.

%In the past several years, techniques combining the  DDP/ ILQG approach with RL techniques \cite{RLHD4, RLHD5, RLHD1} have shown the potential for RL algorithms to scale to continuous high dimensional robotic task planning and learning problems. 
%For continuous state and control space problems, the method of choice is to wrap an LQR feedback policy around a nominal trajectory and then perform a recursive optimization of the feedback law, along with the underlying trajectory, via repeated simulations. However, the parametrization can still be very large and can lead to the so-called `policy chatter' phenomenon \cite{RLHD1}.  

{\bf Our Contributions:}
Rather than directly finding the closed loop control law which requires solving a dynamic programming problem, our approach addresses the original stochastic control problem in a 'decoupled open loop - closed loop' fashion. In this approach: i) we solve an open loop deterministic optimization problem to obtain an optimal nominal trajectory in a model-free fashion, and then  ii) we design a closed loop controller for the resulting linearized time-varying system around the optimal nominal trajectory,  in a model-based fashion. This `divide and conquer' strategy can be shown to be extremely effective. In this context,  our major contributions are: 1) we show a near optimal parametrization of the feedback policy in terms of an open loop control sequence, and a linear feedback control law, 2) we show rigorously that the open loop and closed loop learning can be decoupled, which 3) results in the D2C algorithm that is highly data efficient when compared to  state of the art RL techniques.  This paper is a rigorous extension of our preliminary work \cite{cdc_soc, separation}, in particular, it includes a stronger decoupling result, and an extensive empirical evaluation of the D2C algorithm with state of the art RL implementations on standard benchmark problems.

The rest of the paper is organized as follows. In Section \ref{sec2}, the basic problem formulation is outlined. In Section \ref{sec3}, a decoupling result which solves the MDP in a ``decoupled open loop-closed loop " fashion is briefly summarized. In Section \ref{sec4}, we propose a decoupled data based control algorithm, with discussions of implementation problems. In Section \ref{sec5}, we  test the proposed approach using four typical benchmarking examples with comparisons to a state of the art RL technique.

%% file: Problem_Formulation.tex
%%%%%%%%%%%%%%%%%%%%%%%%%%%%%%%%%%%%%%%%%

\section{PROBLEM FORMULATION}\label{sec2}
Consider the following discrete time nonlinear stochastic dynamical system:
\begin{align} 
\label{original system}
x_{t+1} = h(x_{t}, u_{t},  w_{t}),
\end{align}
where $x_t \in \mathbb{R}^{n_x}$,   $u_t \in \mathbb{R}^{n_u}$ are the state  measurement and control vector at time $k$, respectively. The process noise $w_{t}$ is assumed as zero-mean, uncorrelated Gaussian white noise, with covariance $W$. 
%and  $\epsilon$ is a noise scaling parameter.

The \emph{optimal stochastic control} problem is to find the the control policy $\pi^{o} = \{\pi^{o}_1, \pi^{o}_2, \cdots, \pi^{o}_{T-1} \}$ such that the expected cumulative cost is minimized, i.e., 
\begin{align}
\label{cost_sto_orig}
\pi^{o} &= \arg \min_{\pi}  ~ \tilde{J}^{\pi}(x),~~\text{where,}  \nonumber \\
\tilde{J}^{\pi}(x) &= \mathbb{E}_{\pi} \left[ \sum_{t = 1}^{T-1} c(x_{t}, u_{t}) + c_{T}(x_{T}) | x_{1} = x \right], 
\end{align}
$u_{t} = \pi_{t}(x_{t})$, $c(\cdot, \cdot)$ is the instantaneous cost function, and $c_{T}(\cdot)$ is the terminal cost function.  In the following, we assume that the initial state $x_{1}$ is fixed, and denote $\tilde{J}^{\pi}(x)$ simply as $\tilde{J}^{\pi}$. 

If the function  $h(\cdot, \cdot, \cdot)$ is known exactly, then the optimal control law $\pi^{o}$ can be computed using dynamic programming method. However, as noted before, this can be often computationally intractable. Moreover, when $h$ is unknown, designing an optimal closed loop control law is  a much more  challenging problem. In the following, we propose a data based decoupled approach for solving \eqref{cost_sto_orig} when $h$ is unknown.

%
%
%\subsection{Reinforcement Learning}
%Reinforcement Learning (RL) provides us with a framework to solve Markov Decision Processes (MDPs). Let an MDP be denoted by the tuple $(S_t, A_t, p_{sa}, R_t, \gamma)$, where $S_t$ denotes the current state, $A_t$ is the action, $p_{sa}$ is the transition probability, $R_t$, the corresponding reward and $\gamma$ ($\in [0,1]$) being the discount factor, all at time $t$. The value function for a policy $\pi$ is then given by 
%\begin{equation}
%    V_{\pi}(s) = E[\Sigma^{\infty}_{t=0} \gamma^{t} R(s_t, \pi(s_t))|s_0 = s]
%\end{equation}
%An equivalent formulation can be derived from the Bellman equation to Q-function following policy $\pi$ is as follows:
%\begin{equation*}
%Q_{\pi}(s_0,a_0) =     E[R(s_1,s_0,a_0) + \gamma*Q_{\pi}(s_1, \pi(s_1))]
%\end{equation*}
%
%An optimal policy is one that maximizes the expected cumulative rewards $\pi^{*}(s) = arg\underset{\pi}max$ $  V_{\pi}(s)$ or in terms of Q-function, $\pi^{*}(s) = arg\underset{a}max$ $Q^{*}(s,a)$
% 
% Q-learning is an iterative algorithm to solve MDPs based on the aforementioned Bellman equation that gradually converges the $Q$ state-action value to $Q^{*}$. In deep reinforcement learning, DDPG is one of the successful approaches that employs actor-critic architecture for Q-learning and policy update to dealing with MDPs in continuous state and action spaces. 

%% file: A_Near_Optimal_Decoupling_Principle.tex
%%%%%%%%%%%%%%%%%%%%%%%%%%%%%%%%%%%%%%%%%%%%%%%%%%%%%%%%%%%%%%%%%%%%%%%%%%%%%%%%
\section{A NEAR OPTIMAL DECOUPLING PRINCIPLE}\label{sec3}
We first outline a near-optimal decoupling principle in stochastic optimal control that paves the way for the D2C algorithm  described in Section \ref{sec4}.

We make the following assumptions for the simplicity of illustration.  We assume that the dynamics given in \eqref{original system} can be written in the  form 
\begin{align} 
\label{eq.0.1}
%\label{eq:dynamics1}
x_{t+1} = f(x_t) + B_{t} (u_t + \epsilon w_t), 
\end{align}
where $\epsilon < 1$ is a small parameter. We also assume that the instantaneous cost $c(\cdot, \cdot)$ has the following simple form,  
\begin{align}
\label{eq:cost1}
c(x,u) = l(x) + \frac{1}{2} u'Ru.
\end{align}
We emphasis that these assumptions,  quadratic  control cost and affine in control dynamics, are purely for the simplicity of treatment. These assumptions can be  omitted at the cost of increased notational   complexity.    

%With a slight abuse of notation, let $\pi^{*} = \pi = (\pi_{t})^{T}_{t=1}$ be the optimal control law given by \eqref{cost_sto_orig} according to the dynamics \eqref{eq.0.1} and instantaneous cost \eqref{eq:cost1}.   

\subsection{Linearization w.r.t. Nominal Trajectory}

Consider a noiseless  version of the system dynamics given by \eqref{eq.0.1}. We denote the ``nominal'' state trajectory as $\bar{x}_{t}$ and the ``nominal'' control as $\bar{u}_{t}$ where $u_{t} = \pi_{t}(x_{t})$, where $\pi = (\pi_{t})^{T-1}_{t=1}$ is a given control policy. The resulting dynamics  without noise is given by $\bar{x}_{t+1} = f(\bar{x}_t) + B_{t} \bar{u}_t$.   

%Let the associated nominal incremental costs be denoted by $\bar{c}_t = l(\bar{x}_t) + \frac{1}{2} \bar{u}_t'R\bar{u}_t$, and the associated nominal terminal cost be denoted by $\bar{c}_T = c_T(\bar{x}_T)$. 

Assuming that $f(\cdot)$ and $\pi_{t}(\cdot)$ are sufficiently smooth, we can linearize the dynamics about the nominal trajectory. Denoting  $\delta x_t = x_t - \bar{x}_t, \delta u_t = u_t - \bar{u}_t$, we can express, 
\begin{align}
\delta x_{t+1} &= A_t \delta x_t + B_t \delta u_t + S_t(\delta x_t) + \epsilon w_t, \label{eq.2}\\
\delta u_{t} &=  K_t \delta x_t + \tilde{S}_t(\delta x_t), \label{eq.3}
\end{align}
where $A_t = \frac{\partial f}{\partial x}|_{\bar{x}_t}$, $K_{t} = \frac{\partial \pi_{t}}{\partial x}|_{\bar{x}_t}$, and  $S_t(\cdot), \tilde{S}_t(\cdot)$ are second and higher order terms in the respective expansions. Similarly, we can linearize the instantaneous cost $c(x_{t}, u_{t})$ about the nominal values $(\bar{x}_{t}, \bar{u}_{t})$ as,
\begin{align}
c(x_t,u_t) &= {l}(\bar{x}_{t}) + L_t \delta x_t + H_t(\delta x_t) + \nonumber\\ 
&\hspace{1cm} \frac{1}{2}\bar{u}_t'R\bar{u}_t +  \delta u_t'R\bar{u}_t + \delta u_t'R\delta u_t,\label{eq.4}\\
c_{T}(x_{T}) &= {c}_{T}(\bar{x}_{T}) + C_T \delta x_T + H_T(\delta x_T),\label{eq.5}
\end{align}
where $L_t = \frac{\partial l}{\partial x}|_{\bar{x}_t}$, $C_T  = \frac{\partial c_T}{\partial x}|_{\bar{x}_t}$, and $H_t(\cdot)$ and $H_T(\cdot)$ are second and higher order terms in the respective expansions.

Using \eqref{eq.2} and \eqref{eq.3}, we can write the closed loop dynamics  of the trajectory $(\delta x_{t})^{T}_{t=1}$ as, 
\begin{align}
\label{eq.6}
\delta x_{t+1} = \underbrace{(A_t+B_tK_t)}_{\bar{A}_t} \delta x_t + \underbrace{\{B_t\tilde{S}_t(\delta x_t) + S_t(\delta x_t)\}}_{\bar{S}_t(\delta x_t)} + \epsilon w_t, 
\end{align} 
where $\bar{A}_t$ represents the linear part of the closed loop systems and the term $\bar{S}_t(.)$ represents the second and higher order terms in the closed loop system. Similarly, the closed loop incremental cost given in  \eqref{eq.4} can be expressed as
\begin{align}
 \label{eq.7}
c(x_t,u_t) = \underbrace{\{{l}(\bar{x}_{t}) + \frac{1}{2}\bar{u}_t'R\bar{u}_t\}}_{\bar{c}_t} + \underbrace{[L_t + \bar{u}_t'RK_t]}_{\bar{C}_t} \delta x_t \nonumber\\
+ \underbrace{(K_t\delta x_t+\tilde{S}_t(\delta x_t))'R(K_t\delta x_t + \tilde{S}_t(\delta x_t))}_{\bar{H}_t(\delta x_t)}.
\end{align}

Therefore, the cumulative cost of any given closed loop trajectory $(x_{t}, u_{t})^{T}_{t=1}$  can be expressed as,
\begin{align}
\label{eq.9a}
J^{\pi} &= \sum^{T-1}_{t=1}c(x_{t}, u_{t} = \pi_{t}(x_{t})) + c_{T}(x_{T}) \nonumber \\
&=\sum_{t=1}^T \bar{c}_t + \sum_{t=1}^T \bar{C}_t \delta x_t + \sum_{t=1}^T \bar{H}_t(\delta x_t),
\end{align}
where $\bar{c}_{T} = c_{T}(\bar{x}_{T}),  \bar{C}_{T} = C_{T}$.

We first show the following result. 
\begin{lemma} 
\label{L1}
The state perturbation equation 
\[\delta x_{t+1} = \bar{A}_{t}  \delta x_{t} + \bar{S}_t(\delta x_t) + \epsilon w_t \]
given in  \eqref{eq.6} can be equivalently characterized  as
\begin{align}
\label{eq:mod-pert-1}
\delta x_{t}  = \delta x_t^l + \bar{\bar{S_t}}, ~ \delta x_{t+1}^l = \bar{A}_t \delta x_t^l + \epsilon w_t
\end{align} 
where $\bar{\bar{S}}_t$ is an $O(\epsilon^2)$ function that depends on the entire noise history $\{w_0,w_1,\cdots w_t\}$ and $\delta x_t^l$ evolves according to the  linear closed loop system.
\end{lemma}
Proof is  omitted due to page limits. Detailed proof is given in \cite{d2cTR}.

Using  \eqref{eq:mod-pert-1} in  \eqref{eq.9a}, we can obtain the cumulative cost of any given closed loop trajectory as,
\begin{align}
J^{\pi} = \underbrace{\sum_{t=1}^T \bar{c}_t }_{\bar{J}^{\pi}} + \underbrace{\sum_{t=1}^T \bar{C}_t \delta x_t^l}_{\delta J_1^{\pi}} + %\nonumber\\
\underbrace{\sum_{t=1}^T \bar{H}_t(\delta x_t) + \bar{C}_t \bar{\bar{S}}_t}_{\delta J_2^{\pi}}. \label{eq.9b}
\end{align}

Now, we show the following important result.

\begin{proposition}  
\label{prop1} 
 \begin{align*}
\tilde{J}^{\pi} &= \mathbb{E}[J^{\pi} ] =  \bar{J}^{\pi} + O(\epsilon^2), \\
 \text{Var}(J^{\pi}) &=  \underbrace{\text{Var}(\delta J_{1}^{\pi})}_{O(\epsilon^{2})} +  O(\epsilon^4). 
 \end{align*}
 \end{proposition}
 
\begin{proof}
From \eqref{eq.9b}, we get,
\begin{align} 
\tilde{J}^{\pi} = \mathbb{E}[J^{\pi}] =  \mathbb{E}[ \bar{J}^{\pi} + \delta J_1^{\pi} + \delta J_2^{\pi}], \nonumber\\
= \bar{J}^{\pi} + \mathbb{E}[\delta J_2^{\pi}]  = \bar{J}^{\pi} + O(\epsilon^2), \label{eq.10}
%= \bar{J}_0^{\pi} + \underbrace{ \mathbb{E}[\delta J_2^{\pi}]}_{\delta \tilde{J}_2^{\pi}} = \bar{J}^{\pi}_0 + O(\epsilon^2). \label{eq.10}
\end{align}
The first equality in the last line of the equations before follows from the fact that $\mathbb{E}[\delta x_t^l] = 0$, since its the linear part of the state perturbation driven by white noise and by definition $\delta x_1^l = 0$.The second equality follows form the fact that $\delta J_2^{\pi}$ is an $O(\epsilon^2)$ function.  Now,
\begin{align}
\text{Var}(J^{\pi}) = \mathbb{E}[ J^{\pi} - \tilde{J}^{\pi}]^2 \nonumber\\
= \mathbb{E}[ \bar{J}_0^{\pi} + \delta J_1^{\pi} + \delta J_2^{\pi} - \bar{J}_0^{\pi} - \delta \tilde{J}_2^{\pi}]^2 \nonumber\\
= \text{Var}(\delta J_1^{\pi}) + \text{Var}(\delta J_2^{\pi})  + 2 \mathbb{E}[\delta J_1^{\pi} \delta J_2^{\pi}].
\end{align}
Since $\delta J_2^{\pi}$ is $O(\epsilon^2)$, $\text{Var}(\delta J_2^{\pi})$ is an $O(\epsilon^4)$ function. It can be shown that $\mathbb{E}[\delta J_1^{\pi} \delta J_2^{\pi}]$ is $O(\epsilon^4)$ as well (proof is given \cite{d2cTR}). Finally $\text{Var}(\delta J_1^{\pi})$ is an $O(\epsilon^2)$ function because $\delta x^l$ is an $O(\epsilon)$ function. Combining these, we will get the desired result. 
\end{proof} 

The following observations can now be made from Proposition \ref{prop1}. 
 
\begin{remark}[Expected cost-to-go]Recall that $u_{t} = \pi_t(x_t)$ $= \bar{u}_t + K_t\delta x_t + \tilde{S}_t(\delta x_t)$. However, note that due to Proposition \ref{prop1}, the expected cost-to-go, $\tilde{J}^{\pi}$, is determined almost solely (within $O(\epsilon^2)$)  by the nominal control action sequence $\bar{u}_t$. In other words, the linear and higher order feedback terms have only $O(\epsilon^2)$ effect on the expected cost-to-go function.
\end{remark}

\begin{remark}[Variance of cost-to-go] Given the nominal control action $\bar{u}_t$, the variance of the cost-to-go, which is $O(\epsilon^2)$, is determined overwhelmingly (within $O(\epsilon^4)$) by the linear feedback term $K_t \delta x_t$, i.e., by the variance of the linear perturbation of the cost-to-go, $\delta J_1^{\pi}$, under the linear closed loop system $\delta x_{t+1}^l = (A_t+B_tK_t)\delta x_t^l + \epsilon w_t$.
\end{remark}

\subsection{Decoupled Approach for Closed Loop Control} 

Proposition \ref{prop1} and the remarks above suggest that an open loop control  super imposed with a closed loop control for the perturbed linear system may be approximately optimal. We delineate  this idea below. 

\textit{Open Loop Design.} First, we design an optimal (open loop) control sequence $\bar{u}^{*}_t$ for the noiseless system. More precisely, 
\begin{align}
\label{OL}
(\bar{u}^{*}_t)^{T-1}_{t=1} &= \arg \min_{(\tilde{u}_t)^{T-1}_{t=1}} \sum_{t=1}^{T-1} c(\bar{x}_t, \tilde{u}_t) + c_T(\bar{x}_T), \\
\bar{x}_{t+1} &= f(\bar{x}_t) + B_{t} \tilde{u}_t.  \nonumber
\end{align}
We will discuss the details of this open loop design in Section \ref{sec4}.

\textit{Closed Loop Design.} We find the optimal feedback gain $K^{*}_t$ such that the variance of the linear closed loop system around the nominal path, $(\bar{x}_t, \bar{u}^{*}_t)$, from the open loop design above, is minimized.  
\begin{align}
(K^{*}_t)^{T-1}_{t=1} &=   \arg \min_{(K_t)^{T-1}_{t=1}} ~ \text{Var}(\delta J_1^{\pi}), \nonumber\\
 &\delta J_1^{\pi} = \sum_{t=1}^T \bar{C}_t  x_t^l,\nonumber\\
&\delta x_{t+1}^l = (A_t + B_tK_t) \delta x_t^l + \epsilon w_t. \label{CL}
\end{align} 
We now characterize the approximate closed loop policy below.

\begin{proposition}
Construct a closed loop policy 
\begin{align}
\pi_t^*(x_t) = \bar{u}_t^* + K_t^*\delta x_t,
\end{align} 
where $\bar{u}_t^*$ is the solution of the open loop problem \eqref{OL}, and $K_t^*$ is the solution of the closed loop problem \eqref{CL}. Let $\pi^{o}$ be the optimal closed loop policy. Then, 
\begin{align}
 |\tilde{J}^{\pi*}  - \tilde{J}^{\pi^o}| = O(\epsilon^2).\nonumber
 \end{align}
 Furthermore, among all policies with nominal control action $\bar{u}_t^*$, the variance of the cost-to-go under policy $\pi_t^*$, is within $O(\epsilon^4)$ of the variance of the policy with the minimum variance.
 \end{proposition}
\begin{proof}
We have 
\begin{align*}
\tilde{J}^{\pi^*} - \tilde{J}^{\pi^o}  &= \tilde{J}^{\pi^*} -  \bar{J}^{\pi^*} + \bar{J}^{\pi^*} -  \tilde{J}^{\pi^o}  \\
&\leq \tilde{J}^{\pi^*} -  \bar{J}^{\pi^*} + \bar{J}^{\pi^{o}} -  \tilde{J}^{\pi^o}
\end{align*}
The  inequality above is due the fact that $\bar{J}^{\pi^*} \leq \bar{J}^{\pi^{o}}$, by definition of $\pi^{*}$. Now, using Proposition \ref{prop1}, we have that $|\tilde{J}^{\pi^*} - \bar{J}^{\pi^*}| = O(\epsilon^2)$, and $|\tilde{J}^{\pi^o} -  \bar{J}^{\pi^o}| = O(\epsilon^2)$. Also, by definition, we have $\tilde{J}^{\pi^o} \leq   \tilde{J}^{\pi^*}$. Then, from the above inequality, we get 
\begin{align*}
| \tilde{J}^{\pi^*} - \tilde{J}^{\pi^o} | \leq |\tilde{J}^{\pi^*} -  \bar{J}^{\pi^*} | + | \bar{J}^{\pi^{o}} -  \tilde{J}^{\pi^o} | = O(\epsilon^{2})
\end{align*}
A similar argument holds for the   variance as well. 
\end{proof}

Unfortunately, there is no standard solution to the closed loop problem \eqref{CL} due to the non additive nature of the cost function $\text{Var}(\delta J_1^{\pi})$. Therefore,  we solve a standard LQR problem as a surrogate, and the effect is again one of reducing the variance of the cost-to-go by reducing the variance of the closed loop trajectories.

\textit{Approximate Closed Loop Problem.} We solve the following LQR problem for suitably defined cost function weighting factors $Q_t$, $R_t$:
\begin{align}
\label{eq:acl1}
\min_{(\delta u_t)^{T}_{t=1}} ~&\mathbb{E} [\sum_{t=1}^{T-1} \delta x_t'Q_t\delta x_t + \delta u_t'R_t \delta u_t + \delta x_T'Q_T\delta x_t], \nonumber\\
&\delta x_{t+1} = A_t \delta x_t + B_t\delta u_t + \epsilon w_t.
\end{align}
The solution to the above problem furnishes us a feedback gan $\hat{K}_t^*$ which we can use in the place of the true variance minimizing gain $K_t^*$.  

\begin{remark}
Proposition \ref{prop1} states that the expected cost-to-go of the problem is dominated by the nominal cost-to-go. Therefore, even an open loop policy consisting of simply the nominal control action is within $O(\epsilon^2)$ of the optimal expected cost-to-go. However, the plan with the optimal feedback gain $K_t^*$ is strictly better than the open loop plan in that it has a lower variance in terms of the cost to go. Furthermore, solving the approximate closed loop problem using the surrogate LQR problem, we can expect a lower variance of the cost-to-go function as well.
%which is borne out empirically (albeit we cannot prove it).
\end{remark}

%
%\begin{proposition}  
%\label{prop1} 
%Let the stochastic dynamical system be given as in Eq. \ref{eq.0.1} and the associated cost-to-go over the time horizon $T$ be given by Eq. \ref{eq.0.2}. Let $\pi_t(.)$ be a given feedback policy for the system. Furthermore, assume that the dynamics, incremental and terminal cost functions and the feedback policy are all sufficiently smooth to permit expansions about the nominal trajectory, $(\bar{x}_t, \bar{u}_t)$, as in Eqs. \ref{eq.2}-\ref{eq.5}. Then:\\ 
% $J_0^{\pi} (x_0) = \bar{J}_0^{\pi} + \delta J_1^{\pi} + \delta J_2^{\pi}$, 
% where $\bar{J}_0^{\pi} = \sum_{t=1}^{T-1} c(\bar{x}_t,\bar{u}_t) + c_T(\bar{x}_T)$,
% $\delta J_1^{\pi} = \sum_{t=1}^T \bar{C}_t^x \delta x_t^l$, where $\delta x_{t+1} = \bar{A}_t\delta x_t^l + \bar{S}_t(\delta x_t)$,
% and $\delta J_2^{\pi} = \sum_{t=1}^T \bar{H}_t(\delta x_t) + \bar{C}_t^x \bar{\bar{S}}_t(\delta x_t)$.\\
% Furthermore,
% \begin{align}
% E[J_0^\pi (x_0)] = \tilde{J}_0^{\pi}(x_0) = \bar{J}_0^{\pi} + O(\epsilon^2), \nonumber\\
% Var(J_0^{pi}(x_0)) = \underbrace{Var(\delta J_1^{\pi})}_{O(\epsilon^2)} + O(\epsilon^4). \nonumber
% \end{align}
% \end{proposition}

%%%%%%%%%%%%%%%%%%%%%%%%%%%%%%%%%%%%%%%%%

%% file: Algorithm.tex
\section{Decoupled Data Based Control (D2C) Algorithm}\label{sec4}

In this section, we propose a novel decoupled data based control (D2C) algorithm formalizing  the ideas proposed in Section \ref{sec3}. First, a noiseless open-loop optimization problem is solved to find a nominal optimal trajectory. Then a linearized closed-loop controller is designed around this nominal trajectory, such that, with existence of stochastic perturbations, the state stays close to the optimal open-loop trajectory. The three-step framework to solve the stochastic feedback control problem may be summarized as follows.
\begin{itemize}
\item Solve the open loop optimization problem using gradient decent with a black box simulation model of the dynamics.
\item Linearize the system around the nominal open loop optimal trajectory, and identify the linearized time-varying system from input-output experiment data using a suitable system identification algorithm. 
\item Design an LQR controller which results in an optimal linear control policy around the nominal trajectory. 
\end{itemize}
In the following section, we discuss each of the above steps.

\subsection{Open Loop Trajectory Optimization}
\label{sec_open}
A first order gradient descent based algorithm is proposed here for solving the open loop optimization problem given in \eqref{OL}, where the underlying dynamic model is used as a blackbox, and the necessary gradients are found from a sequence of input perturbation experiment data using standard least square.

Denote  the initial guess of the control sequence as $U^{(0)} = \{\bar{u}_t^{(0)} \}_{t = 1}^{T}$, and  the corresponding  states $\mathcal{X}^{(0)} = \{\bar{x}_t^{(0)} \}_{t=1}^T$. The control policy is updated iteratively via
\begin{align}\label{ctrl_upd}
U^{(n + 1)} = U^{(n)} - \alpha \nabla_U \bar{J}|_{\mathcal{X}^{(n)}, U^{(n)}},
\end{align}
where $U^{(n)} = \{\bar{u}_t^{(n)} \}_{t = 1}^{T}$ denotes the control sequence in the $n^{th}$ iteration, $\mathcal{X}^{(n)} = \{\bar{x}_t^{(n)}\}_{t=1}^{T}$ denotes the corresponding states, and $\alpha$ is the step size parameter. As $\bar{J}|_{\mathcal{X}^{(n)}, U^{(n)}}$ is the expected cumulative cost under control sequence $U^{(n)}$ and corresponding states $\mathcal{X}^{(n)}$, the gradient vector is defined as
\begin{align}\label{gradv}
\nabla_U \bar{J}|_{\mathcal{X}^{(n)}, U^{(n)}} = \begin{pmatrix} \frac{\partial \bar{J}}{\partial u_1} & \frac{\partial \bar{J}}{\partial u_1} & \cdots & \frac{\partial \bar{J}}{\partial u_{T}} \end{pmatrix}|_{\mathcal{X}^{(n)}, U^{(n)}},
\end{align}
which is the gradient of the expected cumulative cost w.r.t the control sequence after $n$ iterations. The following paragraph elaborates on how to estimate the above gradient. 

Let us define a rollout to be an episode in the simulation that starts from the initial settings to the end of the horizon with a control sequence. For each iteration, multiple rollouts are conducted sequentially with both the expected cumulative cost and the gradient vector updated iteratively after each rollout. During one iteration for the control sequence, the expected cumulative cost is calculated as
\begin{align}\label{cost_ite}
\bar{J}|_{\mathcal{X}^{(n)}, U^{(n)}}^{j+1} = 
(1-\frac{1}{j})\bar{J}|_{\mathcal{X}^{(n)}, U^{(n)}}^{j} +  \frac{1}{j}(J|_{\mathcal{X}^{j,(n)}, U^{j,(n)}}),
\end{align}
where $j$ denotes the $j^{th}$ rollout within the current iteration process of control sequence. $\bar{J}|_{\mathcal{X}^{(n)}, U^{(n)}}^{j}$ is the expected cumulative cost after $j$ rollouts while $J|_{\mathcal{X}^{j,(n)}, U^{j,(n)}}$ denotes the cost of the $j^{th}$ rollout under control sequence $U^{j,(n)}$ and corresponding states $\mathcal{X}^{j,(n)}$. Note that $U^{j,(n)} = \{\bar{u}_t^{(n)} + \delta u^{j,(n)}_t\}_{t = 1}^{T}$ where $\{\delta u^{j,(n)}_t\}^T_{t=1}$ is the zero-mean, i.i.d Gaussian noise added as perturbation to the control sequence $U^{(n)}$.

Then the gradient vector is calculated in a similar sequential manner as
\begin{align}\label{grad_ite}
\nabla_U\bar{J}|_{\mathcal{X}^{(n)}, U^{(n)}}^{j+1} = 
(1-\frac{1}{j})\nabla_U\bar{J}|_{\mathcal{X}^{(n)}, U^{(n)}}^{j} + \nonumber\\ \frac{1}{j\sigma_{\delta u}}(J|_{\mathcal{X}^{j,(n)}, U^{j,(n)}}-\bar{J}|_{\mathcal{X}^{(n)},U^{(n)}}^{j+1})(U^{j,(n)}-U^{(n)}),
\end{align}
where $\sigma_{\delta u}$ is the variance of the control perturbation and $\nabla_U\bar{J}|_{\mathcal{X}^{(n)}, U^{(n)}}^{j+1}$ denotes the gradient vector after $j$ rollouts. Note that after each rollout, both the expected cumulative cost and the gradient vector are updated. The rollout number $m$ in one iteration for the control sequence is decided by the convergence of both the expected cumulative cost and the gradient vector. After $m$ rollouts, the control sequence is updated by equation (\ref{ctrl_upd}) in which $\nabla_U \bar{J}|_{\mathcal{X}^{(n)}, U^{(n)}}$ is estimated by $\nabla_U\bar{J}|_{\mathcal{X}^{(n)}, U^{(n)}}^{m}$. Keep doing this until the cost converges and the optimized nominal control sequence is $ \{\bar{u}^*_t\}_{t = 1}^{T}=\{\bar{u}^{(N-1)}_t\}_{t = 1}^{T}$.

Higher order approaches other than gradient descent are possible. However, for a general system, the gradient descent approach is easy to implement. Also it is memory efficient and highly amenable to parallelization as a result of our sequential approach.

\subsection{Linear Time-Varying System Identification}
\label{sec_ls}
Closed loop control design  specified in \eqref{CL} or the approximate closed loop control design specified in \eqref{eq:acl1} requires the knowledge of the parameters $A_{t}, B_{t}, 1 \leq t \leq T,$ of the perturbed linear system. We propose a linear time variant (LTV) system identification procedure to estimate these parameters.  

First start from perturbed linear system given by equation \eqref{eq:acl1}. Using only first order information and estimate the system parameters $A_{t}, B_{t}$ with the following form
\begin{align}
\label{est_sys}
& \delta x_{t+ 1} = \hat{A_t} \delta x_t+\hat{B_t} \delta u_t,
\end{align}
rewrite with augmented matrix
\begin{align}
\label{aug system}
& \delta x_{t+ 1} = [\hat{A_t} ~|~ \hat{B_t}]  ~
\begin{bmatrix}
   \delta x_t \\
   \delta u_t \\
\end{bmatrix},
\end{align}
Now write out their components for each iteration in vector form as,
\begin{align}
\label{lsq1}
& Y = [\delta x_{t+1}^0 \delta x_{t+1}^1 \cdots \delta x_{t+1}^{N-1}],\nonumber\\
& X=\begin{bmatrix}
   \delta x_t^0 & \delta x_t^1& \cdots & \delta x_t^{N-1} \\
   \delta u_t^0 &  \delta u_t^1 & \cdots & \delta u_t^{N-1} \\
\end{bmatrix},\nonumber\\
&  Y=[\hat{A_t} ~|~ \hat{B_t}]X,
\end{align}
where N is the total iteration number. $\delta x_{t+1}^n$ denotes the output state deviation , $\delta x_{t}^n$ denotes the input state perturbations and $\delta u_{t}^n$ denotes the input control perturbations at time $t$ of the $n^{th}$ iteration. All the perturbations are zero-mean, i.i.d, Gaussian random vectors whose covariance matrix is $\sigma I$ where  I is the identity matrix and $\sigma $ is a scalar. Note that here one iteration only has one rollout.

The next step is to apply the perturbed control $\{\bar{u}^*_t + \delta u_t^n\}_{t = 1}^{T}$ to the system and collect input-output experiment data in a blackbox simulation manner.

Using the  least square method $\hat{A_t}$ and $\hat{B_t}$ can be calculated in the following procedure
\begin{align}
\label{lsq2}
&  YX'=[\hat{A_t} ~|~ \hat{B_t}]XX',
\end{align}
As the perturbations are zero-mean, i.i.d, Gaussian random noise, $XX' = \sigma I$. Then
\begin{align}
\label{lsq3}
 &[\hat{A_t} ~|~ \hat{B_t}] = \frac{1}{\sigma}YX' \nonumber \\
 &=\frac{1}{\sigma}[\delta x_{t+1}^0 \delta x_{t+1}^1 \cdots \delta x_{t+1}^{N-1}]
\begin{bmatrix}
   (\delta x_t^0)' &  (\delta u_t^0)'\\
   (\delta x_t^1)' &  (\delta u_t^1)'\\
   \vdots & \vdots \\
   (\delta x_t^{N-1})' &  (\delta u_t^{N-1})'\\
\end{bmatrix}
\end{align}

The calculation procedure can also be done in a sequential way similar to the update of the gradient vector in the open-loop optimization algorithm. Therefore it is highly amenable to parallelization and memory efficient.

\subsection{Closed Loop Control Design} 
Given the parameter estimate of the perturbed linear system, we solve the  closed loop control problem given in \eqref{eq:acl1}. This is a standard LQR problem. By solving the Riccati equation, we can get the closed-loop optimal feedback gain $K^*_t$. The details of the design is standard and is omitted here. 

\begin{algorithm}
    \caption{D2C Algorithm}
    \label{algo1}
   {\bf 1)}  Solve the deterministic open-loop optimization problem for optimal open loop nominal control sequence and trajectory $(\{\bar{u}^*_t\}_{t = 1}^{T}, \{\bar{x}_t^*\}_{t = 1}^T)$ using gradient descent method (Section \ref{sec_open}).\\
  {\bf 2)} Identify the LTV system $(\hat{A}_t, \hat{B}_t)$ via least square estimation (Section \ref{sec_ls}).\\
  {\bf 3)} Solve the Riccati equations  using estimated LTV system equation for feedback gain $\{ K^*_t \}_{t = 1}^{T}$.\\
  {\bf 4)} Set $t = 1$, given initial state $x_1 = \bar{x}_1^*$ and state deviation $\delta x_1 = 0$.\\
\While{$t \leq T$}{
  \begin{align}
& u_t = \bar{u}_t^* + K^*_t \delta x_t, \nonumber \\
& x_{t + 1} = f(x_{t}) + B_{t} (u_{t} + \epsilon w_{t}),\nonumber \\
& \delta x_{t+1} = x_{t + 1}-\bar{x}^*_{t + 1}
\end{align}
 $t = t + 1$.
}
\end{algorithm}

\subsection{D2C Algorithm: Summary}
The Decoupled Data Based Control  (D2C) Algorithm is summarized in Algorithm \ref{algo1}.

%% file: Simulations.tex
\section{Simulation Results}\label{sec5}
In this section, we compare the D2C approach with the well-known deep reinforcement learning algorithm - Deep Deterministic Policy Gradient (DDPG)  \cite{DDPG}. For the comparison, we evaluate both the methods in the following three aspects:  
\begin{itemize}
\item Data efficiency in training - the amount of data sampling and storage required to achieve a desired task.
\item Robustness to noise - the deviation from the predefined task due to random noise in process in the testing stage.
\item Ease of training - the challenges involved in training with either of the data-based approaches.
\end{itemize}

\subsection{Tasks and Implementation}
We tested our method with four  benchmark tasks, all implemented in MuJoCo simulator \cite{mujoco}: Inverted pendulum, Cart-pole, 3-link swimmer and 6-link swimmer \cite{dpmd_suite}. Each of the systems and their tasks are briefly defined as follows: 

\paragraph{Inverted pendulum} A swing-up task of this 2D system from its downright initial position is considered. 

\paragraph{Cart-pole} 
The state of a 4D under-actuated cart-pole comprises of the angle of the pole, cart's horizontal position and their rates. Within a given horizon, the task is to swing-up the pole and balance it at the middle of the rail by applying a horizontal force on the cart.
\paragraph{3-link Swimmer} 
The 3-link swimmer model has 5 degrees of freedom and together with their rates, the system is described by 10 state variables. The task is to solve the planning and control problem from a given initial state to the goal position located at the center of the ball. Controls can only be applied in the form of torques to two joints. Hence, it is under-actuated by 3 DOF.
\paragraph{6-link Swimmer} 
The task with a 6-link swimmer model is similar to that defined in the 3-link case. However, with 6 links, it has 8 degrees of freedom and hence, 16 state variables, controlled by 5 joint motors.

\begin{figure*}[!htb]

\begin{multicols}{4}
    \hspace{0.5cm}    
    \subfloat[Inverted Pendulum - D2C]{\includegraphics[width=1\linewidth]{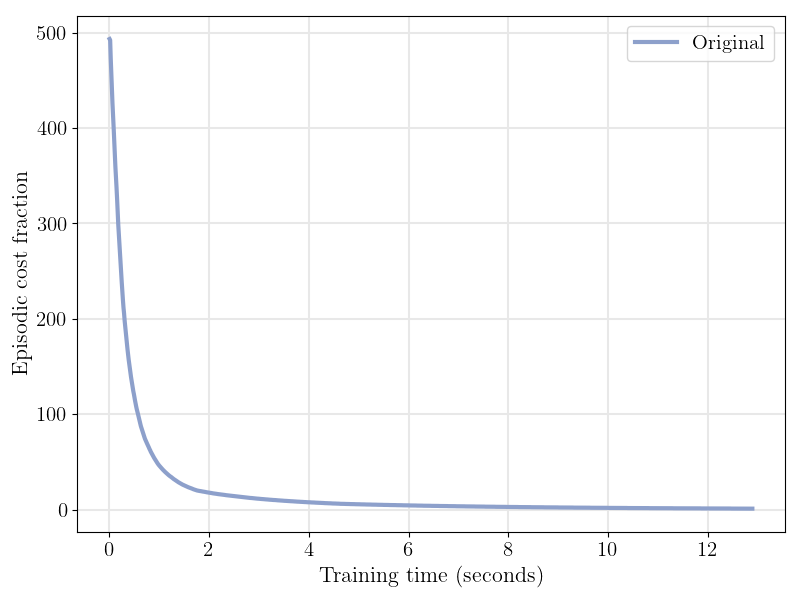}}
     \subfloat[Cart-Pole - D2C]{\includegraphics[width=\linewidth]{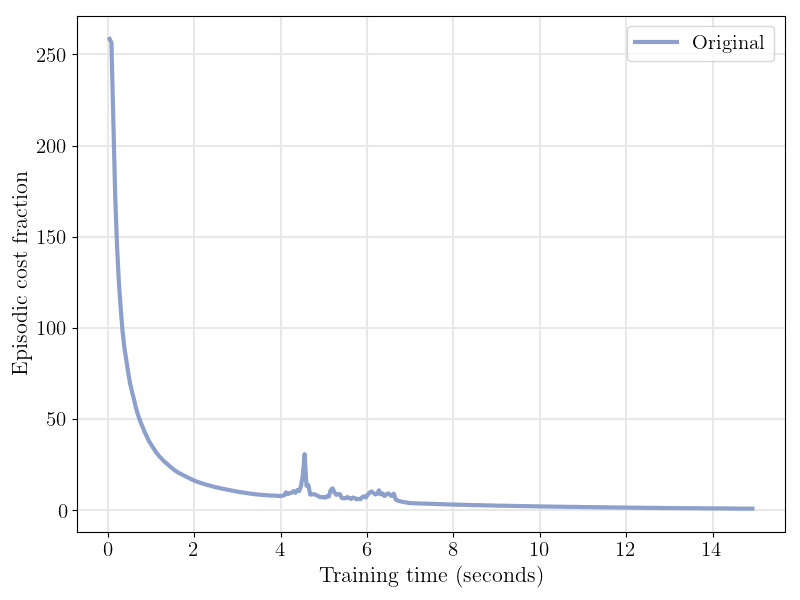}}
     \subfloat[3-link Swimmer - D2C]{\includegraphics[width=\linewidth]{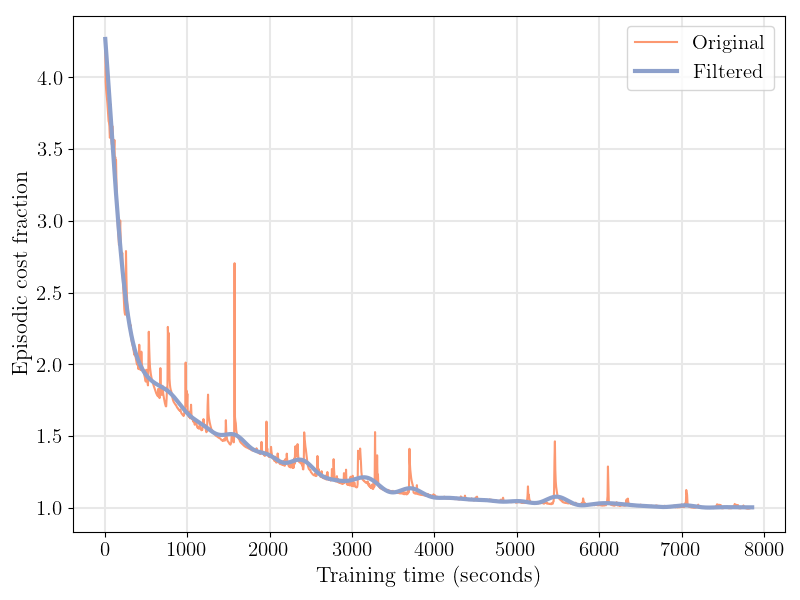}}
     \subfloat[6-link Swimmer - D2C]{\includegraphics[width=\linewidth]{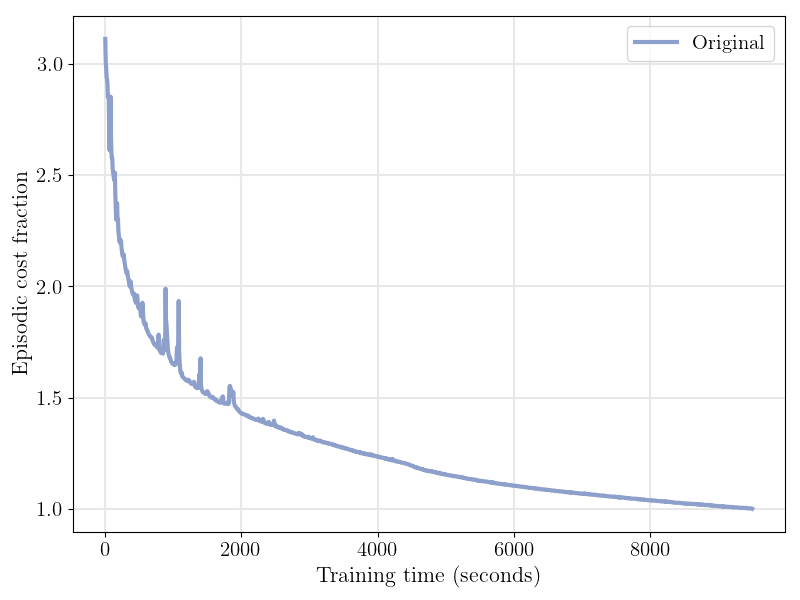}}
     
    \end{multicols}
\begin{multicols}{4} 
    \hspace{0.5cm}    
      \subfloat[Inverted Pendulum - DDPG]{\includegraphics[width=\linewidth]{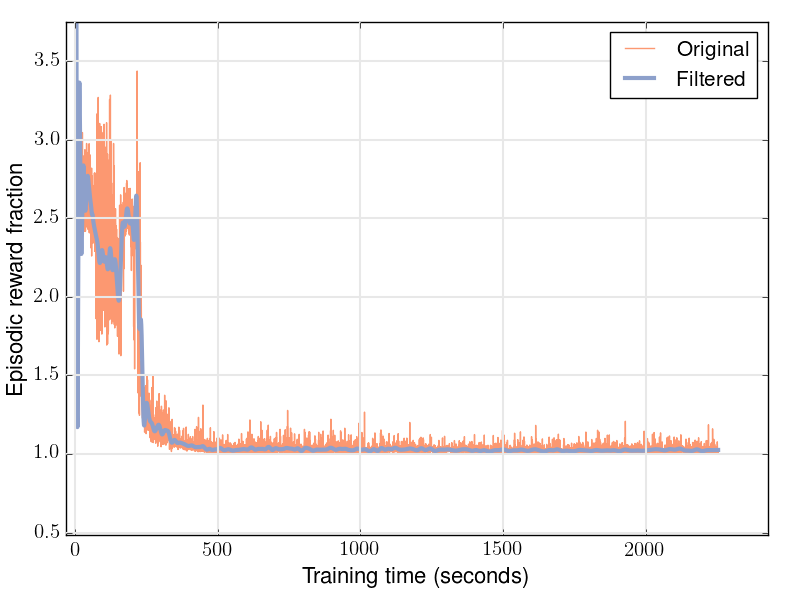}}
      \subfloat[Cart-Pole - DDPG]{\includegraphics[width=\linewidth]{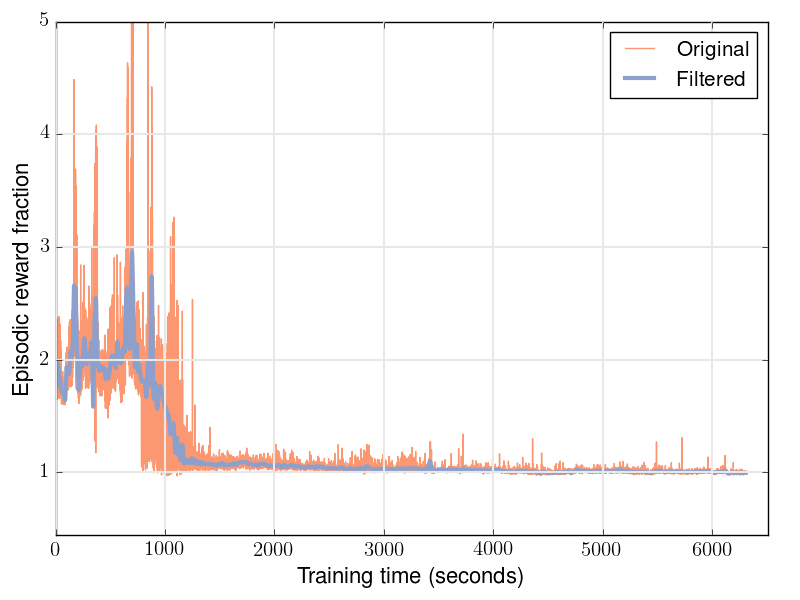}}
      \subfloat[3-link Swimmer - DDPG]{\includegraphics[width=\linewidth]{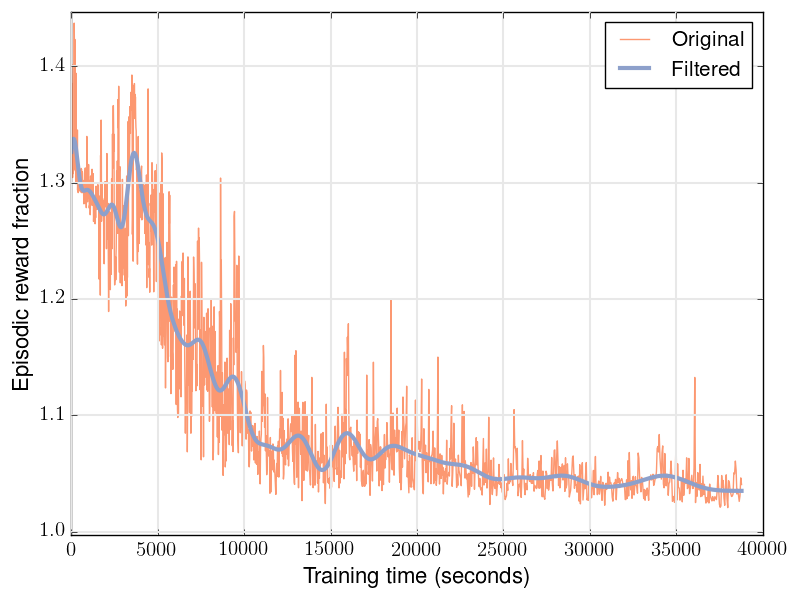}}
      \subfloat[6-link Swimmer - DDPG]{\includegraphics[width=\linewidth]{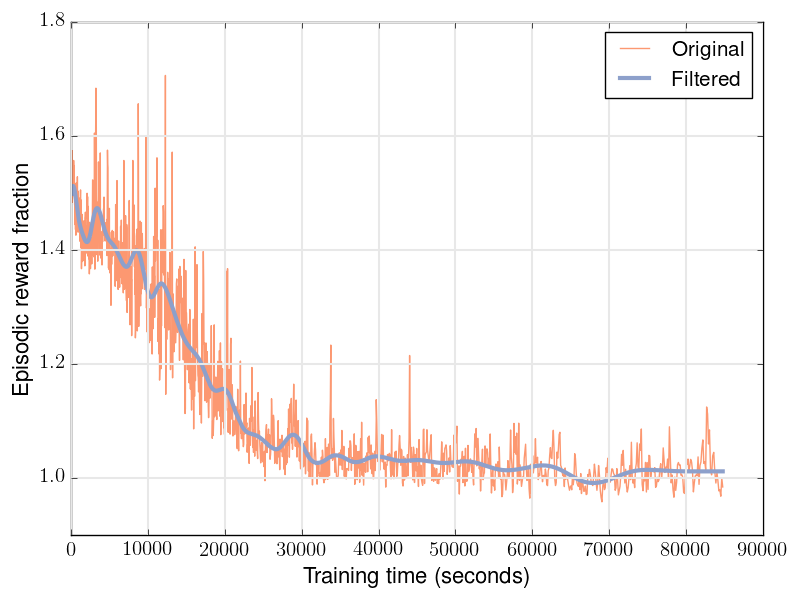}}
\end{multicols}
\caption{Episodic reward fraction vs time taken during training}
\end{figure*}

\begin{figure*}[!htb]
\begin{multicols}{4}
    \hspace{0.5cm}    
      \subfloat[Inverted Pendulum]{\includegraphics[width=\linewidth]{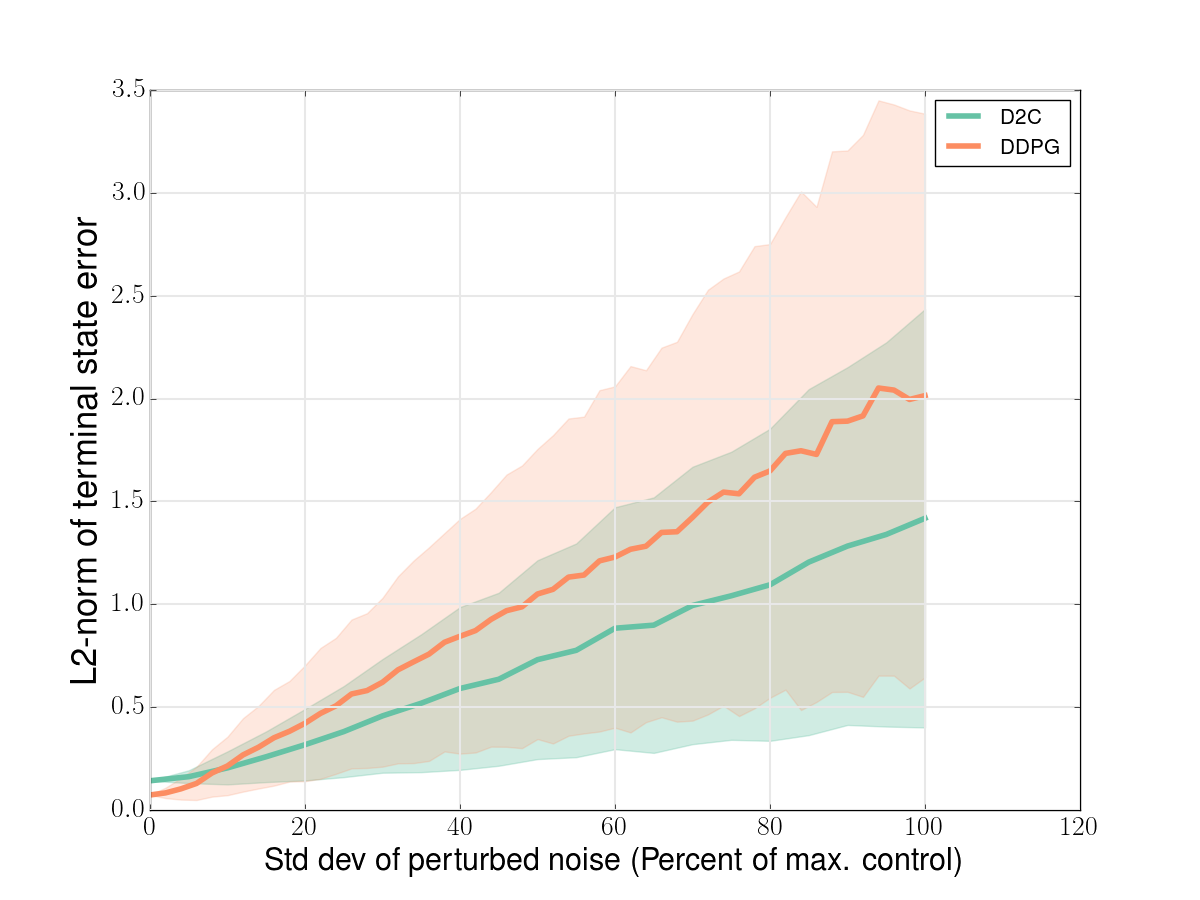}}    
      \subfloat[Cart-Pole ]{\includegraphics[width=\linewidth]{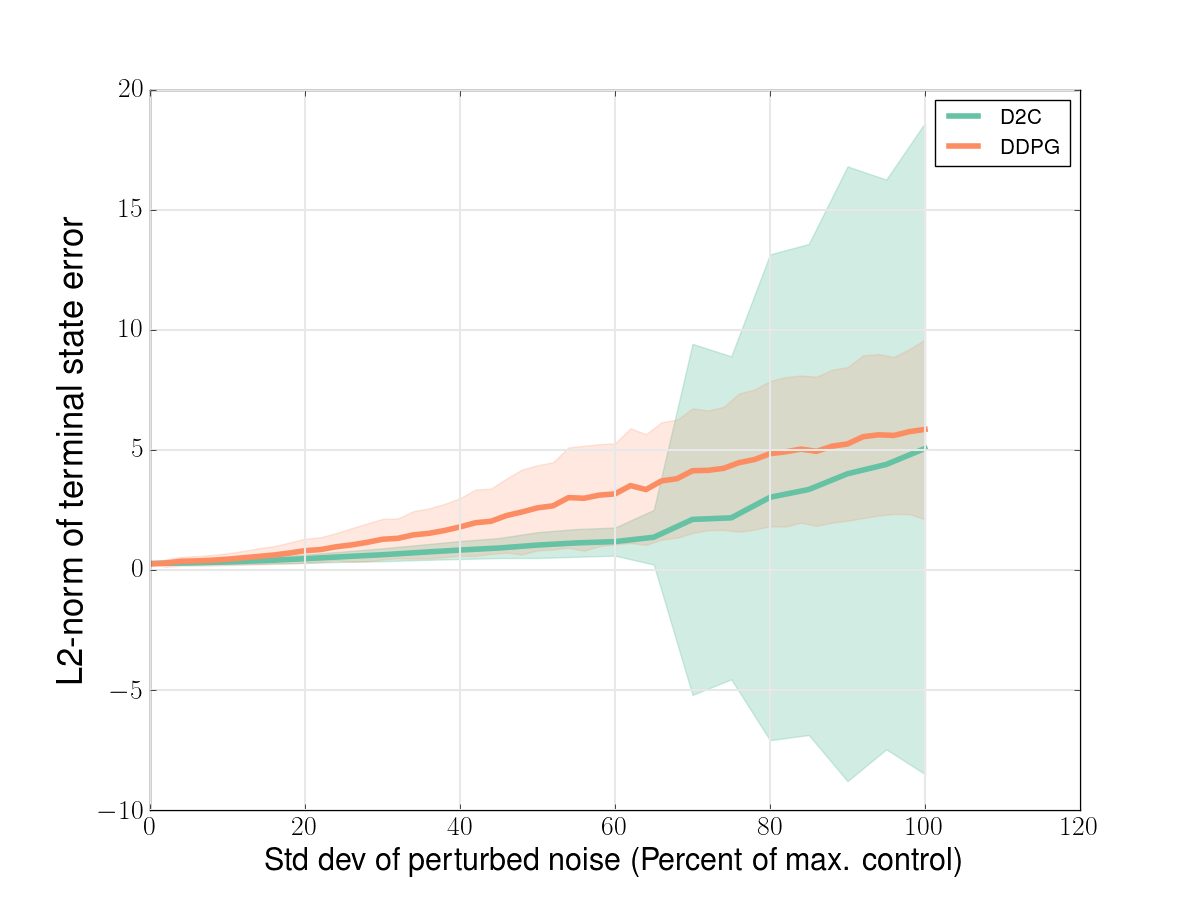}}
      \subfloat[3-link Swimmer]{\includegraphics[width=\linewidth]{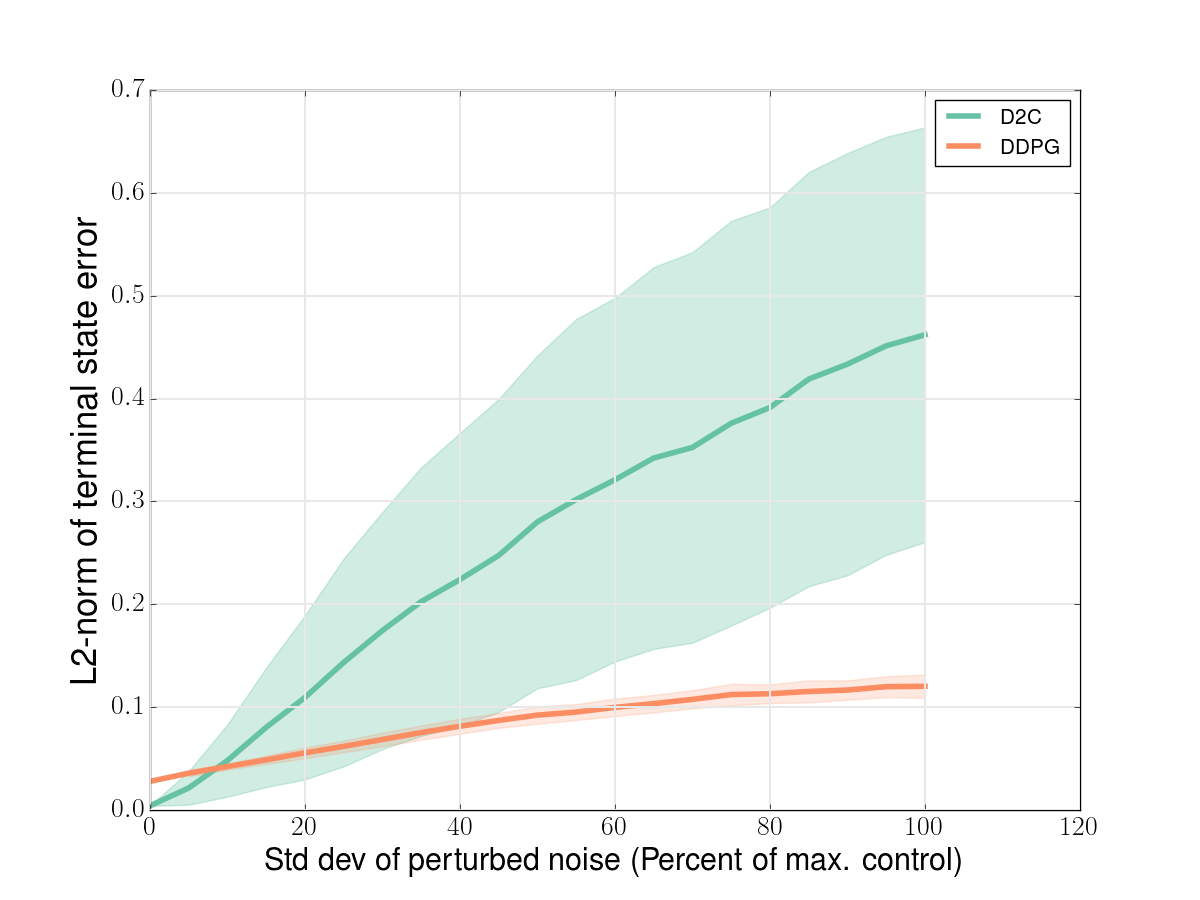}}
      \subfloat[6-link Swimmer ]{\includegraphics[width=\linewidth]{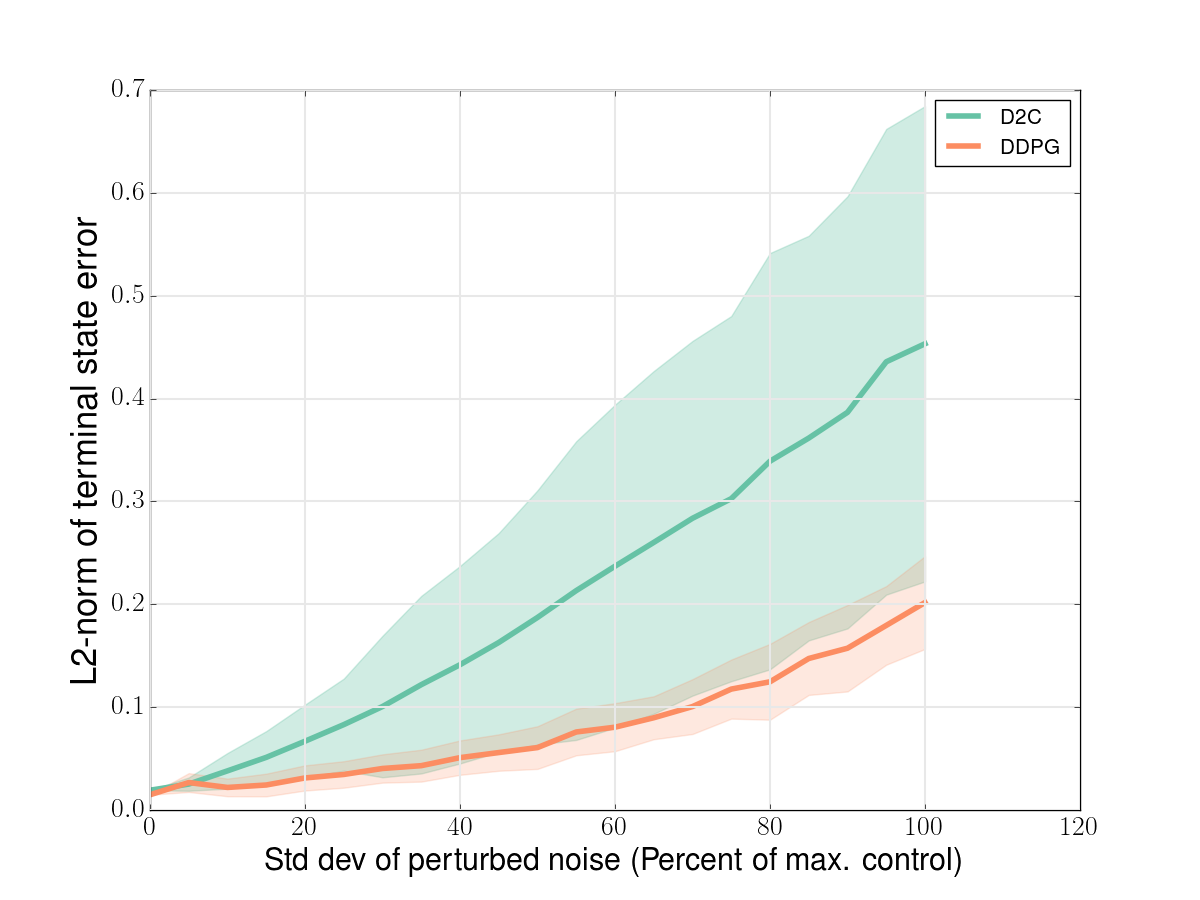}}
\end{multicols}
\caption{Terminal MSE vs noise level during testing}
\end{figure*}

D2C implementation is done in three stages corresponding to those mentioned in the previous section and `MuJoCo Pro 150' is used as the environment for simulating the blackbox model. An off-the-shelf implementation of DDPG provided by {\it Keras-RL} \cite{plappert2016kerasrl} library has been customized for our simulations. It may be noted that the structure of the reward function is formulated to optimize the performance of DDPG and hence, different from that used in D2C. However, the episode length (or horizon) and the time-discretization step is held identical. Although continuous RL algorithms such as DDPG learn the policy and thereby interpolating it to the entire state space, we consider the task-based experiments in a finite limited time-frame window approach. Training is performed until the neural networks' loss is converged. Hence, though the episodic reward does seem to converge earlier, that itself may not indicate that the policy is converged. 

For fair comparison, `Episodic reward/cost fraction' is considered with both methods. It is defined as the fraction of reward obtained in an episode during training w.r.t the nominal episodic reward (converged reward). Note that the words 'reward' and 'cost' are being used interchangeably due to their different notions in optimal control and RL literature respectively, though they achieve the same objective. For simplicity, one is considered the negative of the other.

\subsection{Performance Comparison}

{\bf Data-efficiency:} As mentioned above, an efficient training is one that requires minimal data sampling in order to achieve the same behavior. One way of measuring it is to collate the times taken for the episodic cost (or reward) to converge during training. Plots in Fig. 1 show the training process with both methods on the systems considered. Table \ref{timecompare} delineates the times taken for training respectively. As the system identification and feedback gain calculation in case of D2C take only a small portion of time, the total time comparison in (Table \ref{timecompare}) shows that D2C learns the optimal policy substantially faster  than DDPG and hence, has a better data efficiency. 

{\bf Robustness to noise:} As with canonical off-policy RL algorithms, DDPG requires that an exploration noise be added to the policy, during training. Given that the training adapts the policy to various levels of noise, combined with hours of intense training and a nonlinear policy output, it is expected that it is more robust towards noise as is evident from Figs. 2 (c) and 2 (d). However, from plots in Figs. 2 (a) and (b), it is evident that in some systems the performance of D2C is on par with or better than DDPG. 
It may also be noted that the error variance in D2C method increases abruptly when the noise level is higher than a threshold and drives the system too far away from the nominal trajectory that the LQR controller cannot fix it. This could be considered as a drawback for D2C. However, it must be noted that the range of noise levels (up until 100 \% of the maximum control signal) that we are considering here is far beyond what is typically encountered in practical scenarios. Hence, even in swimmer examples, the performance of D2C is tolerable to a reasonable extent of noise in the system.

{\bf Ease of training:} The ease of training is often an ignored topic in analyzing a reinforcement learning (RL) algorithm. By this, we mean the challenges associated with its implementation. As with many RL algorithms that involve neural networks, DDPG has no guarantees for policy convergence. As a result, the implementation often involves tuning a number of hyper-parameters and a careful reward shaping in a trial and error fashion, which is even more time-consuming given that their successful implementation already involves significant time and computational resources. Reproducibility \cite{henderson2018deep} is still a major challenge that RL is yet to overcome. On the other hand, training in our approach is predictable and very reliable.

To elucidate the ease of training from an empirical perspective, the exploration noise that is required for training in DDPG mandates the system to operate with a shorter time-step than a threshold, beyond which the simulation fails due to an unbearable magnitude of control actions into the system. For this, we train both the swimmers in one such case (with $\Delta t = 0.01$ sec) till it fails and execute the intermediate policy. Fig. \ref{testing_mse_001_swimmers} shows the plot in the testing-stage with both methods. It is evident from the terminal state mean-squared error at zero noise level that the nominal trajectory of DDPG is incomplete and its policy failed to reach the goal. The effect is more pronounced in the higher-dimensional 6-link swimmer system (Fig. \ref{testing_mse_001_swimmer6}), where the DDPG's policy can be deemed to be downright broken. Note, from Table \ref{timecompare}, that the systems have been trained with DDPG for a time that is more than thrice with the 3-link swimmer and 4 times with the 6-link swimmer. On the other hand, under same conditions, the seamless training of D2C results in a working policy with even greater data-efficiency.

%The fundamental differences in both the algorithms arise from their distinct formulations, as in, finite-horizon in case of D2C and an infinite-horizon problem in DDPG. From optimal control literature, it is well-known that the terminal cost in a finite-horizon problem plays a crucial role in stabilizing the system \cite{}. Hence, methods such as D2C are able to achieve their goals accurately whereas DDPG consume inordinate amount of time in 'fine-tuning' their behavior towards the goal. However, we also note DDPG explores over the entire state-space and can result in a generic policy. Another drawback with D2C over canonical RL algorithms is that the cost could be stuck in a local minimum, whereas DDPG due to the nature of stochastic gradient descent in training neural networks, can potentially reach the global optimal solution. Nevertheless, we hope that our approach signifies the potential of decoupling based approaches such as D2C in a reinforcement learning paradigm and recognizes the need for more hybrid approaches that complement the merits of each.

\begin{table}[h]
\caption{Simulation parameters and training outcomes}
\label{timecompare}
\begin{center}
\begin{tabular}{|c|c|c|c|c|c|}
\hline
System&Steps per& Time-& \multicolumn{3}{|c|}{Training time (in sec.)}\\\cline{4-6}
&episode&step&\multicolumn{2}{|c|}{D2C}& \\\cline{4-5}
&&(in sec.)& Open-& Closed-&DDPG\\
&&&loop&loop&\\
\hline
Inverted& 30 & 0.1 &12.9  &$<0.1$& 2261.15\\
Pendulum& & & & &\\
\hline
Cart pole& 30 & 0.1 & 15.0 & 1.33& 6306.7\\
\hline
3-link & 1600& 0.005 & 7861.0&13.1 & 38833.64\\
Swimmer&  800& 0.01 & 4001.0&4.6 & 13280.7*\\
\hline
6-link &1500 & 0.006 &9489.3&26.5 & 88160\\
Swimmer &900 & 0.01 &3585.4&16.4 & 15797.2*\\ 
\hline
\end{tabular}
\end{center}
\end{table}

\begin{figure}
     \centering
     \subfloat[3-link swimmer]
     {\includegraphics[width=0.23\textwidth]{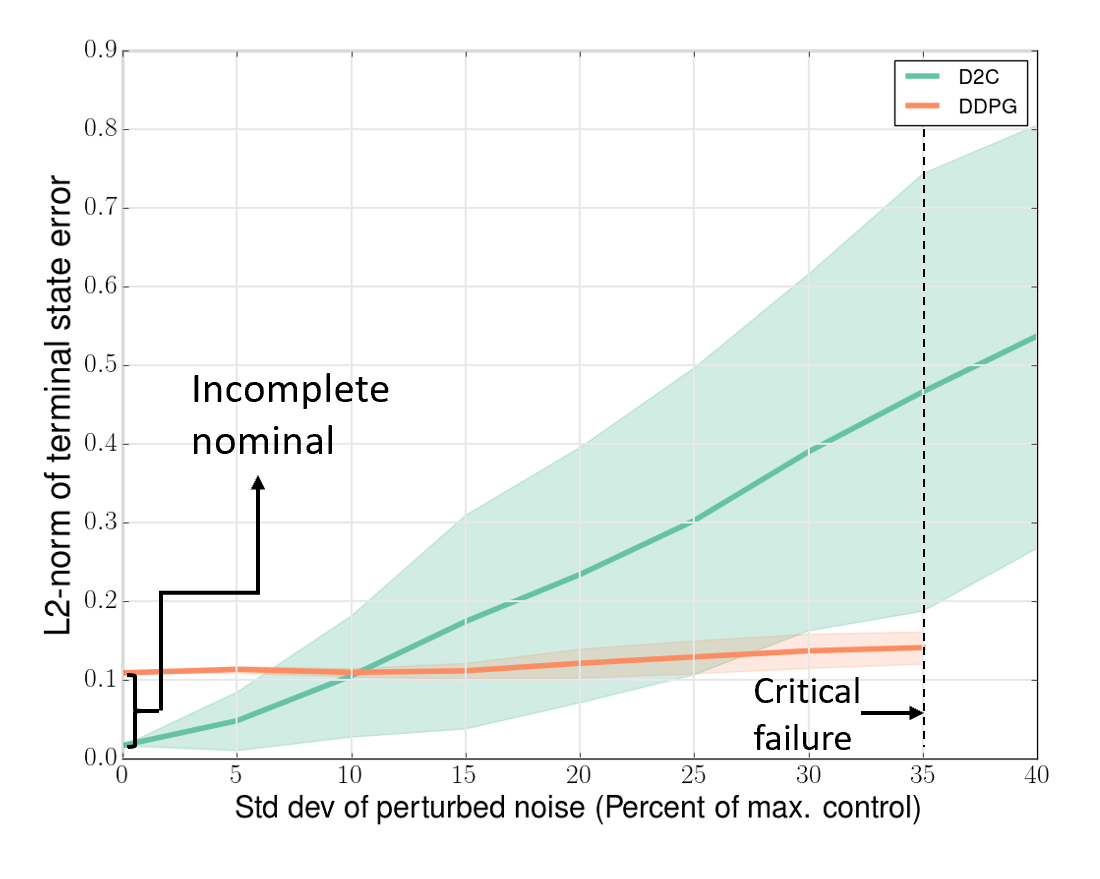}
     \label{testing_mse_001_swimmer}}
     \subfloat[6-link swimmer]
     {\includegraphics[width=0.23\textwidth]{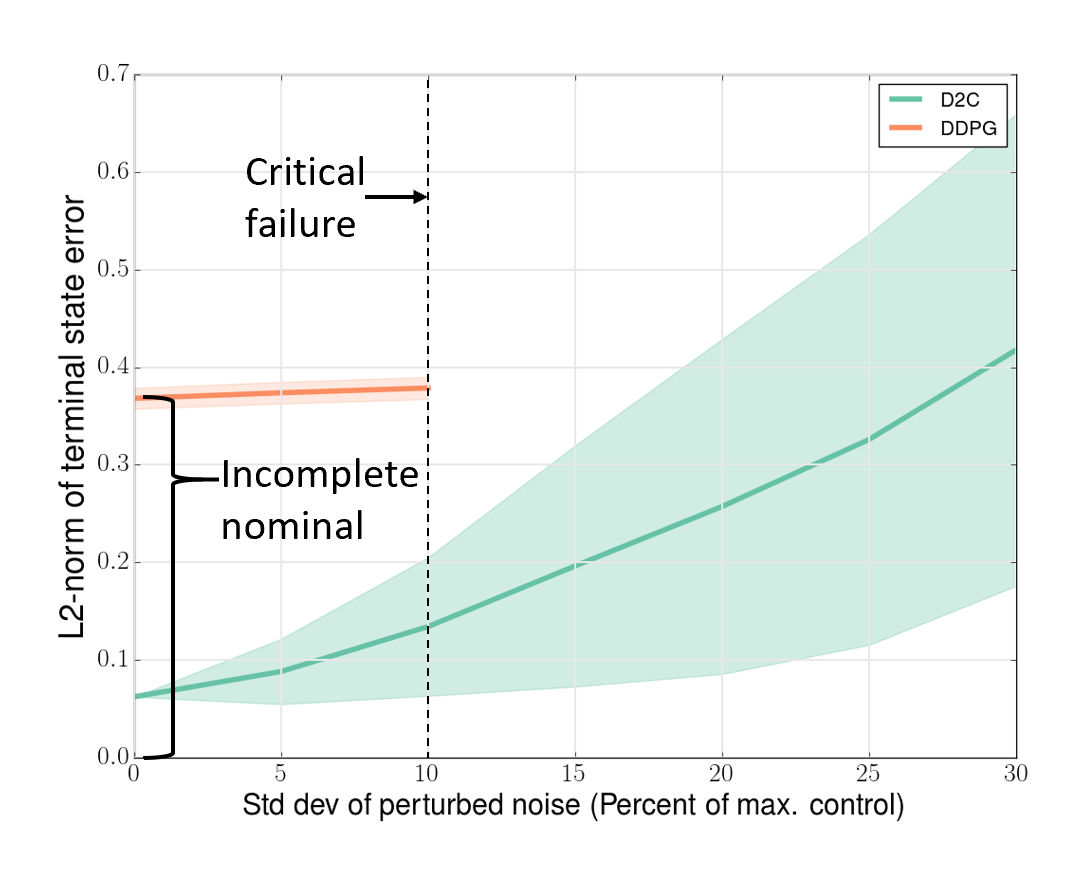}
     \label{testing_mse_001_swimmer6}}
     \caption{D2C vs DDPG at $\Delta t = 0.01 s$}
     \label{testing_mse_001_swimmers}
\end{figure}

%% file: Conclusions.tex
\section{CONCLUSIONS}
In this paper, we proposed a near-optimal control algorithm under fully observed conditions and showed that our method is able to scale-up to higher dimensional state-space without any knowledge about the system model. Due to sequential calculation used in the open-loop optimization and the system identification, D2C is highly memory efficient and also convenient for parallelization. We tested its performance and juxtaposed them with a state-of-the-art deep RL technique - DDPG. From the results, our method has conspicuous advantages over DDPG in terms of data efficiency and ease of training. The robustness of D2C is also better in some cases, but has scope for further development by employing more sophisticated feedback methods and ensuring that the data efficiency is not compromised. We also believe further drastic reduction in the planning time can be achieved by parallelization and a more sophisticated parametrization of the open loop problem.

It is evident from the simulations that methods such as D2C are able to achieve their goals accurately whereas DDPG consumes inordinate amount of time in `fine-tuning' their behavior towards the goal. However, we also note that, by doing this, DDPG is tentatively exploring over the entire state-space and can result in a better generic policy. Another drawback with D2C over canonical RL algorithms is that the cost could be stuck in a local minimum, whereas DDPG due to the nature of stochastic gradient descent in training neural networks, can potentially reach the globally optimal solution. Nevertheless, we hope that our approach signifies the potential of decoupling based approaches such as D2C in a reinforcement learning paradigm and recognizes the need for more hybrid approaches that complement the merits of each.

%% file: Bibliography.tex
\bibliographystyle{IEEEtran}
\bibliography{CDC_refs}

%% file: Appendix.tex
\newpage

\section*{APPENDIX}

\subsection{Proof of Lemma \ref{L1}}
\label{sec:proofL1}
%\label{sec:proofL1}

\begin{proof}
We proceed by induction. The first general instance of the recursion occurs at $t=3$.
It can be shown that: 
 \begin{align}
 &\delta x_3 = \underbrace{(\bar{A}_2\bar{A}_1(\epsilon w_0) + \bar{A}_2 (\epsilon w_1) + \epsilon w_2)}_{\delta x_3^l} + \nonumber\\
 &\underbrace{\{\bar{A}_2 \bar{S}_1(\epsilon w_0) + \bar{S}_2(\bar{A}_1(\epsilon w_0) + \bar{S}_2(\bar{A}_1(\epsilon w_0) + \epsilon w_1 + \bar{S}_1(\epsilon w_0))\}}_{\bar{\bar{S}}_3}. 
 \end{align}
 Noting that $\bar{S}_1(.)$ and $\bar{S}_2(.)$ are second and higher order terms, it follows that $\bar{\bar{S}}_3$ is $O(\epsilon^2)$. \\
 Suppose now that $\delta x_t = \delta x_t^l + \bar{\bar{S}}_t$ where $\bar{\bar{S}}_t$ is $O(\epsilon^2)$. Then:
 \begin{align}
 \delta x_{t+1} = \bar{A}_{t+1}(\delta x_t^l + \bar{\bar{S_t}}) + \epsilon w_t + \bar{S}_{t+1}(\delta x_t), \nonumber\\
 = \underbrace{(\bar{A}_{t+1} \delta x_t^l + \epsilon w_t)}_{\delta x_{t+1}^l} +\underbrace{\{\bar{A}_{t+1}\bar{\bar{S}}_t + \bar{S}_{t+1}(\delta x_t)\}}_{\bar{\bar{S}}_{t+1}}.
 \end{align}
 Noting that $\bar{S}_{t+1}$ is $O(\epsilon^2)$ and that $\bar{\bar{S}}_{t+1}$ is $O(\epsilon^2)$ by assumption, the result follows. 
\end{proof}

\subsection{Lemma 2}
\label{sec:proofL2}

\begin{lemma} 
\label{L2}
Let $\delta J_1^{\pi}$, $\delta J_2^{\pi}$ be as defined in \eqref{eq.9b}. Then, $\mathbb{E} [\delta J_1 \delta J_2]$ is an $O(\epsilon^4)$ function.
\end{lemma}

\begin{proof}
In the following, we suppress the explicit dependence on $\pi$ for $\delta J_1^{\pi}$ and $\delta J_2^{\pi}$ for notational convenience.
Recall that $\delta J_1 = \sum_{t=0}^T c_t^x \delta x_t^l$, and $\delta J_2 = \sum_{t=0}^T \bar{H}_t(\delta x_t) + c_t^x \bar{\bar{S}}_t$.  For notational convenience, let us consider the scalar case, the vector case follows readily at the expense of more elaborate notation. Let us first consider $\bar{\bar{S}}_2$. We have that $\bar{\bar{S}}_2 = \bar{A}_2\bar{S}_1(\epsilon w_0) + \bar{S}_2(\bar{A}_1(\epsilon w_0) + \epsilon w_1+ \bar{S}_1(\epsilon w_0))$. Then, it follows that:
\begin{align}
\bar{\bar{S}}_2 = \bar{A}_2 \bar{S}_1^{(2)}(\epsilon w_0)^2 + \bar{S}_2^{(2)}(\bar{A}_1 \epsilon w_0 + \epsilon w_1)^2 + O(\epsilon^3),
\end{align}
where $\bar{S}_t^{(2)}$ represents the coefficient of the second order term in the expansion of $\bar{S}_t$. A similar observation holds for $H_2(\delta x_2)$ in that:
\begin{align}
\bar{H}_2(\delta x_2) = \bar{H}_2^{(2)}(\bar{A}_1(\epsilon w_0) + \epsilon w_1)^2 + O(\epsilon^3),
\end{align}
where $\bar{H}_t^{(2)}$ is the coefficient of the second order term in the expansion of $\bar{H}_t$. Note that $\epsilon w_0 = \delta x_1^l$ and $\bar{A}_1(\epsilon w_0) + \epsilon w_1= \delta x_2^l$. Therefore, it follows that we may write:
\begin{align}
\bar{H}_t(\delta x_t) + C_t^x \bar{\bar{S}}_t = \sum_{\tau = 0}^{t-1} q_{t,\tau}(\delta x_{\tau}^l)^2 + O(\epsilon^3),
\end{align}
for suitably defined coefficients $q_{t,\tau}$. 
Therefore, it follows that 
\begin{align}
\delta J_2 = \sum_{t=1}^T \bar{H}_t(\delta x_t) + C_t^x \bar{\bar{S}}_t\nonumber\\
= \sum_{\tau = 0}^T \bar{q}_{T,\tau}(\delta x_{\tau}^l)^2+ O(\epsilon^3),
\end{align}
for suitably defined $\bar{q}_{T,\tau}$. Therefore:
\begin{align}
\delta J_1 \delta J_2 = \sum_{t,\tau} C_{\tau}^x(\delta x_{\tau}^l)\bar{q}_{T,t}(\delta x_t^l)^2 + O(\epsilon^4).
\end{align}
Taking expectations on both sides:
\begin{align}
E[\delta J_1 \delta J_2] = \sum_{t,\tau} C_{\tau}^x \bar{q}_{T,t} E[\delta x_{\tau}^l (\delta x_t^l)^2] + O(\epsilon^4).
\end{align}
Break $\delta x_t^l = (\delta x_t^l - \delta x_{\tau}^l) + \delta x_{\tau}^l$, assuming $\tau < t$. Then, it follows that:
\begin{align}
E[\delta x_{\tau}^l (\delta x_t^l)^2] = E[\delta x_{\tau}^l (\delta x_t^l - \delta x_{\tau}^l)^2] + E[(\delta x_{\tau}^l)^3]  \nonumber\\
+ 2 E[(\delta x_t^l - \delta x_{\tau}^l)(\delta x_{\tau}^l)^2]\nonumber\\
= E[(\delta x_{\tau}^l)^3],
\end{align}
where the first and last terms in the first equality drop out due to the independence of the increment $(\delta x_t^l - \delta x_{\tau}^l)$ from $\delta x_{\tau}^l$, and the fact that $E[\delta x_t^l - \delta x_{\tau}^l] = 0$ and $E[\delta x_{\tau}^l] = 0$. Since $\delta x_{\tau}^l$ is the state of the linear system $\delta x_{t+1}= \bar{A}_t \delta x_t^l + \epsilon w_t$, it may again be shown that:
\begin{align}
E[\delta x_{\tau}^l]^3 = \sum_{s_1,s_2,s_3} \Phi_{\tau, s_1}\Phi_{\tau,s_2}\Phi_{\tau,s_3} E[w_{s_1}w_{s_2}w_{s_3}],
\end{align}
where $\Phi_{t,\tau}$ represents the state transitions operator between times $\tau$ and $t$, and follows from  the closed loop dynamics. Now, due to the independence of the noise terms $w_t$, it follows that $E[w_{s_1}w_{s_2}w_{s_3}] = 0$ regardless of $s_1,s_2,s_3$.\\
 An analogous argument as above can be repeated for the case when $\tau > t$. Therefore, it follows that $E[\delta J_1 \delta J_2] = O(\epsilon^4)$.
\end{proof}

%% file: Main.bbl
\begin{thebibliography}{10}
\providecommand{\url}[1]{#1}
\csname url@samestyle\endcsname
\providecommand{\newblock}{\relax}
\providecommand{\bibinfo}[2]{#2}
\providecommand{\BIBentrySTDinterwordspacing}{\spaceskip=0pt\relax}
\providecommand{\BIBentryALTinterwordstretchfactor}{4}
\providecommand{\BIBentryALTinterwordspacing}{\spaceskip=\fontdimen2\font plus
\BIBentryALTinterwordstretchfactor\fontdimen3\font minus
  \fontdimen4\font\relax}
\providecommand{\BIBforeignlanguage}[2]{{%
\expandafter\ifx\csname l@#1\endcsname\relax
\typeout{** WARNING: IEEEtran.bst: No hyphenation pattern has been}%
\typeout{** loaded for the language `#1'. Using the pattern for}%
\typeout{** the default language instead.}%
\else
\language=\csname l@#1\endcsname
\fi
#2}}
\providecommand{\BIBdecl}{\relax}
\BIBdecl

\bibitem{kumar2015stochastic}
P.~R. Kumar and P.~Varaiya, \emph{Stochastic systems: Estimation,
  identification, and adaptive control}.\hskip 1em plus 0.5em minus 0.4em\relax
  SIAM, 2015, vol.~75.

\bibitem{ioannou2012robust}
P.~A. Ioannou and J.~Sun, \emph{Robust adaptive control}.\hskip 1em plus 0.5em
  minus 0.4em\relax Courier Corporation, 2012.

\bibitem{aastrom2013adaptive}
K.~J. {\AA}str{\"o}m and B.~Wittenmark, \emph{Adaptive control}.\hskip 1em plus
  0.5em minus 0.4em\relax Courier Corporation, 2013.

\bibitem{sastry2011adaptive}
S.~Sastry and M.~Bodson, \emph{Adaptive control: stability, convergence and
  robustness}.\hskip 1em plus 0.5em minus 0.4em\relax Courier Corporation,
  2011.

\bibitem{silver2016mastering}
D.~Silver, A.~Huang, C.~J. Maddison, A.~Guez, L.~Sifre, G.~Van Den~Driessche,
  J.~Schrittwieser, I.~Antonoglou, V.~Panneershelvam, M.~Lanctot \emph{et~al.},
  ``Mastering the game of go with deep neural networks and tree search,''
  \emph{nature}, vol. 529, no. 7587, p. 484, 2016.

\bibitem{lillicrap2015continuous}
T.~P. Lillicrap, J.~J. Hunt, A.~Pritzel, N.~Heess, T.~Erez, Y.~Tassa,
  D.~Silver, and D.~Wierstra, ``Continuous control with deep reinforcement
  learning,'' \emph{arXiv preprint arXiv:1509.02971}, 2015.

\bibitem{levine2016end}
S.~Levine, C.~Finn, T.~Darrell, and P.~Abbeel, ``End-to-end training of deep
  visuomotor policies,'' \emph{The Journal of Machine Learning Research},
  vol.~17, no.~1, pp. 1334--1373, 2016.

\bibitem{acktr}
W.~Yuhuai, M.~Elman, L.~Shun, G.~Roger, and B.~Jimmy, ``Scalable trust-region
  method for deep reinforcement learning using kronecker-factored
  approximation,'' \emph{arXiv:1708.05144}, 2017.

\bibitem{trpo}
S.~John, L.~Sergey, M.~Philipp, J.~Michael~I., and A.~Pieter, ``Trust region
  policy optimization,'' \emph{arXiv:1502.05477}, 2017.

\bibitem{ppo}
S.~John, W.~Filip, D.~Prafulla, R.~Alec, and K.~Oleg, ``Proximal policy
  optimization algorithms,'' \emph{arXiv:1707.06347}, 2017.

\bibitem{henderson2018deep}
P.~Henderson, R.~Islam, P.~Bachman, J.~Pineau, D.~Precup, and D.~Meger, ``Deep
  reinforcement learning that matters,'' in \emph{Thirty-Second AAAI Conference
  on Artificial Intelligence}, 2018.

\bibitem{gu2016q}
S.~Gu, T.~Lillicrap, Z.~Ghahramani, R.~E. Turner, and S.~Levine, ``Q-prop:
  Sample-efficient policy gradient with an off-policy critic,'' \emph{arXiv
  preprint arXiv:1611.02247}, 2016.

\bibitem{dp_num}
M.~Falcone, ``{Recent Results in the Approximation of Nonlinear Optimal Control
  Problems},'' in \emph{Large-Scale Scientific Computing LSSC}, 2013.

\bibitem{ddp}
D.~Jacobsen and D.~Mayne, \emph{Differential Dynamic Programming}.\hskip 1em
  plus 0.5em minus 0.4em\relax Elsevier, 1970.

\bibitem{sddp}
E.~Theoddorou, Y.~Tassa, and E.~Todorov, ``{Stochastic Differential Dynamic
  Programming},'' in \emph{Proceedings of American Control Conference}, 2010.

\bibitem{ilqg1}
E.~Todorov and W.~Li, ``{A generalized iterative LQG method for locally-optimal
  feedback control of constrained nonlinear stochastic systems},'' in
  \emph{Proceedings of American Control Conference}, 2005, pp. 300 -- 306.

\bibitem{ilqg2}
W.~Li and E.~Todorov, ``Iterative linearization methods for approximately
  optimal control and estimation of non-linear stochastic system,''
  \emph{International Journal of Control}, vol.~80, no.~9, pp. 1439--1453,
  2007.

\bibitem{powell2007approximate}
W.~B. Powell, \emph{Approximate Dynamic Programming: Solving the curses of
  dimensionality}.\hskip 1em plus 0.5em minus 0.4em\relax John Wiley \& Sons,
  2007.

\bibitem{bertsekas2012dynamic}
D.~Bertsekas, \emph{Dynamic programming and optimal control}.\hskip 1em plus
  0.5em minus 0.4em\relax Athena scientific Belmont, MA, 2012, vol.~2.

\bibitem{sutton2018reinforcement}
R.~S. Sutton and A.~G. Barto, \emph{Reinforcement learning: An
  introduction}.\hskip 1em plus 0.5em minus 0.4em\relax MIT press, 2018.

\bibitem{cdc_soc}
D.~Yu, M.~Rafieisakhaei, and S.~Chakravorty, ``{Stochastic Feedback Control of
  Systems with Unknown Nonlinear Dynamics},'' in \emph{56$^{th}$ IEEE
  Conference on Decision and Control(CDC)}, 2017.

\bibitem{separation}
M.~Rafieisakhaei, S.~Chakravorty, and P.~R. Kumar, ``{A Near-Optimal Separation
  Principle for Nonlinear Stochastic Systems Arising in Robotic Path Planning
  and Control},'' in \emph{56$^{th}$ IEEE Conference on Decision and
  Control(CDC)}, 2017.

\bibitem{d2cTR}
R.~Wang, K.~Parunandi, D.~Yu, D.~Kalathil, and S.~Chakravorty, ``Decoupled data
  based approach for learning to control nonlinear dynamical systems,''
  \emph{Tech. Report. Available at https://goo.gl/pd3pRU}, March, 2019.

\bibitem{DDPG}
T.~Lillicrap \emph{et~al.}, ``Continuous control with deep reinforcement
  learning,'' in \emph{Proc. ICLR}, 2016.

\bibitem{mujoco}
T.~Emanuel, E.~Tom, and Y.~Tassa, ``Mujoco: A physics engine for model-based
  control,'' \emph{IEEE/RSJ International Conference on Intelligent Robots and
  Systems}, pp. 5026--5033, 2012.

\bibitem{dpmd_suite}
T.~Yuval and {\it et al.}, ``Deepmind control suite,'' \emph{arXiv:1801.00690},
  2018.

\bibitem{plappert2016kerasrl}
M.~Plappert, ``keras-rl,'' \url{https://github.com/keras-rl/keras-rl}, 2016.

\end{thebibliography}
